%% file: main.tex
\documentclass[runningheads]{llncs}

\usepackage[misc,geometry]{ifsym}

\usepackage{amsfonts}
\usepackage{mathtools}
\usepackage{amsmath}

\usepackage{amssymb}
\usepackage{amsthm}
\usepackage{multirow}
\usepackage{float}
\usepackage{xcolor}

\usepackage{bbm} 
\usepackage[ruled,linesnumbered]{algorithm2e}
\usepackage[scr]{rsfso}
\usepackage{nicefrac}
\usepackage{enumitem}

\usepackage{wrapfig}
\usepackage{hyperref}

\DeclareMathOperator{\tr}{tr} 

\usepackage[T1]{fontenc}
\usepackage{graphicx}
\definecolor{umich}{HTML}{000000}

\begin{document}
\title{Quantifying Aleatoric and Epistemic Dynamics Uncertainty via Local Conformal Calibration}
\titlerunning{ Quantifying Dynamics Uncertainty via Local Conformal Calibration}
%
\author{Lu\'is Marques\textsuperscript{(\Letter)}\orcidID{0000-0001-7211-411X} \and Dmitry Berenson\orcidID{0000-0002-9712-109X}}
\authorrunning{L. Marques and D. Berenson}
%
\institute{Robotics Department, University of Michigan, Ann Arbor, MI 48109, USA\\
\email{\{lmarques,dmitryb\}@umich.edu}}
\maketitle              
\begin{abstract}
Whether learned, simulated, or analytical, approximations of a robot's dynamics can be inaccurate when encountering novel environments. Many approaches have been proposed to quantify the aleatoric uncertainty of such methods, i.e. uncertainty resulting from stochasticity, however these estimates alone are not enough to properly estimate the uncertainty of a model in a novel environment, where the actual dynamics can change. Such changes can induce epistemic uncertainty, i.e. uncertainty due to a lack of information/data. Accounting for \textit{both} epistemic and aleatoric dynamics uncertainty in a theoretically-grounded way remains an open problem. We introduce \textbf{L}ocal \textbf{U}ncertainty \textbf{C}onformal \textbf{Ca}libration (LUCCa), a conformal prediction-based approach that calibrates the aleatoric uncertainty estimates provided by dynamics models to generate probabilistically-valid prediction regions of the system's state. We account for both epistemic and aleatoric uncertainty non-asymptotically, without strong assumptions about the form of the true dynamics or how it changes. The calibration is performed locally in the state-action space, leading to uncertainty estimates that are useful for planning. We validate our method by constructing probabilistically-safe plans for a double-integrator under significant changes in dynamics. Companion website: \href{https://um-arm-lab.github.io/lucca}{\textbf{\color{umich}https://um-arm-lab.github.io/lucca}}.
\end{abstract}
\section{Introduction}
Optimal control with known dynamics is effective on a wide range of domains and is often applied to safety-critical tasks requiring provable guarantees \cite{devCDC21,xu2015CBF}. However, the underlying physical interactions of a robot can be difficult to model, leading to the use of simplifying assumptions which compromise fidelity and/or generality. For example, the no-slip assumption of unicycle dynamics may be appropriate when moving slowly over dry pavement, but not at high speeds or when the surface is slippery. Likewise, analytical quadrotor models might not accurately capture ground and near-wall aerodynamic effects, but are otherwise sufficient for free-space flying. While adding stochasticity to the model (e.g. by assuming Gaussian perturbations) can partially capture the true predictive uncertainty, when there are significant changes in the underlying physical interactions with the environment this approach can underrepresent the effect of model mismatch. Further, the heterogeneity of the stochastic perturbations does not necessarily match how the model error varies across the state-action space. Learned approximations of robot dynamics, in conjunction with or in place of analytical models, have been effective for tasks such as quadrotor landing \cite{shi19_lander} and locomotion on mud \cite{schoellig19ECC}. Unfortunately, learned systems can perform arbitrarily badly when out-of-distribution (OOD), making safe deployment in new scenarios challenging. This fundamental limitation of data-driven approaches is exacerbated in fast-changing and complex scenarios (e.g. autonomous driving), where there are numerous edge-cases and the distribution of data can shift between training and inference. Quantifying the uncertainty of learned predictors usually requires imposing limiting assumptions (e.g. zero-mean normal error distribution), provides only asymptotic bounds or relies on heuristics. Provable finite-sample guarantees are lacking. \\ 
\indent Conformal Prediction (CP) is a statistical framework that enables using the outputs of a predictor to construct prediction sets that are valid in a frequentist sense. Given a set of \textit{calibration data} in a new scenario, CP builds a set in the prediction output space that is guaranteed to contain the true label with at least a user-specified probability. Crucially, it makes no assumptions about the predictor's structure or the distribution of the data, and its guarantees are non-asymptotic. While typically tested in low-dimensional synthetic data, CP has the potential to mitigate many of the limitations of dynamics predictors. \\ \indent This paper proposes \textbf{L}ocal \textbf{U}ncertainty \textbf{C}onformal \textbf{Ca}libration (LUCCa), a probabilistically-valid conformal prediction-based approach that accounts for epistemic uncertainty by calibrating the aleatoric uncertainty estimates provided by dynamics predictors. Unlike other CP methods \cite{lindemann2023safe,sun2022copula}, we do not use a global calibration factor, as this tends to be overly conservative and uninformative. Instead our scaling factors are state and action dependent, thus accounting for changes in the approximate dynamics accuracy. LUCCa provably constructs calibrated regions for the first planning step, and, if the approximate dynamics are linear and we can make certain assumptions about the controller, also for all subsequent steps.
We provide analysis showing the validity of our approach and present experiments where LUCCa is used to construct probabilistically-safe plans for a double-integrator under significant dynamics changes. Our contributions are the following: 
\begin{itemize}[nolistsep,topsep=-\parskip]
    \item We propose an algorithm (LUCCa) that, given a dynamics predictor and a small calibration dataset, provides probabilistically valid prediction regions for the robot's future states accounting for both aleatoric and epistemic uncertainty.
    \item We prove the validity of our approach for any finite set of calibration data, a predictor which outputs a multivariate normal uncertainty, any unknown true dynamics function, and uncharacterized aleatoric perturbations.
    \item We calibrate the uncertainty locally relative to the system's state-action space, leading to prediction regions representative of predictive uncertainty.
    \item We demonstrate the effectiveness of our method through experiments on a double-integrator system.
\end{itemize}
\vspace{-0.5em}
\section{Related Work} \label{sec:related_work}
\vspace{-6pt}
 While many dynamical propagation algorithms provide uncertainty estimates, these are often not valid in a frequentist sense. In \cite{kchua2018}, the authors account for aleatoric uncertainty through a learned Multivariate Normal (MVN) predictor and for epistemic uncertainty through ensembles. In \cite{loiannoTRO24} a Gaussian representing aleatoric uncertainty is propagated throughout planning via the Unscented Transform, and in \cite{schoellig19ECC} Gaussian Processes are used to quickly capture changes in dynamics uncertainty. While all these approaches are practically useful, our method could calibrate their outputs to make them also probabilistically valid.

Recently, local CP theory \cite{guan2019,han2022split} has enabled the use of feature-dependent calibration factors, which has the potential to create prediction regions that are more \textit{adaptive} (with size representative of prediction difficulty) and \textit{sharp} (tight). This could also increase plan feasibility and provide the planning algorithm information about the accuracy of its dynamics model. However, these tools are often ill-suited for motion planning as they can produce spuriously large prediction regions, have a reduced local effect in higher-dimensional spaces or require large online computations. We utilize the LOCART method \cite{dtree_brazil}, which can be pre-checked for unbounded prediction regions and does not induce significant online computation overhead.

In robotics, CP has been used to increase the reliability of LLM-planners \cite{askHelp23} and to safely navigate around dynamical agents \cite{dixit23a}. More similar to our work, in \cite{lindemann2023safe,sun2022copula} the authors use CP to provide tight coverage guarantees over multi-step predictions. Conversely to both works, we assume access to a dataset of exchangeable transition tuples instead of full trajectories, thereby removing any implicit assumption about how prediction accuracy is dependent on the current time index. While previous work considered deterministic dynamics and constructed calibrated regions by expanding point predictions equally in all state dimensions (generating hyperspheres), we consider stochastic dynamics and introduce an uncertainty-aware non-conformity score which enables the construction of hyperellipsoidal calibrated regions. Further, instead of providing a calibration factor that is state and action independent, which is often too conservative for motion planning, we capture the variation of epistemic uncertainty across the state-action space.

\vspace{-10pt}
\section{Problem Statement}
\vspace{-6pt}
Let $s_t\in \mathcal S$ and $u_t\in \mathcal U$ respectively denote the robot's state and control input at time ${t}\in \mathbb N_0=\mathbb N\cup \{0\}$, where $\mathcal S$ and $\mathcal U$ represent the admissible regions. For convenience, define $X_t := (s_t, u_t)$, $Y_t := s_{t+1}$, the transition pair as ${Z_t:=(X_t, Y_t)}$ and the sets $\mathcal X := \mathcal S \times \mathcal U$ and $\mathcal Y:= \mathcal S$. We consider discrete-time stochastic systems and describe their time-invariant dynamics by the map $f\colon \bar{\mathcal X}  \mapsto \mathcal P_{Y\mid X}$ such that \vspace{-6pt}
\begin{equation}
    \label{eq:dynamics}
    Y_t \sim f(\bar{X}_t) \quad
    \vspace{-3pt}
\end{equation} 
where $\bar{\mathcal X} :=\mathcal P_{S} \times \mathcal U$, $\bar{X}_t \in \bar {\mathcal X}$ is a tuple with a distribution of the current state and an action, and $\mathcal P$ is the set of distributions that admit a density. We make no assumptions on the smoothness of $f$, which may even be discontinuous, and allow the aleatoric uncertainty in $Y$ to be heteroscedastic. In practice, $f$ is unknown and we only have access to an approximation $\tilde f$, which could be analytical, simulated, or learned through observations of state transitions. Hence, we never directly observe the output of $f$ (the full distribution), but solely a sparse signal of it (a single output for a single input). Generally, $\exists X_t$ s.t $d_{W_1}(\tilde{f}(X_t), f(X_t)) > \rho$ for a fixed small $\rho >0$, where $d_{W_1}$ is the 1-Wasserstein distance, and thus, without imposing any assumptions on the bounds or distribution of the approximation error, our approximate model is not directly suitable for provably safe planning. \\ \indent Furthermore, the true Markovian dynamics of the robot also depend implicitly on the underlying environment. Consider for example that the frictional contact underlying wheeled robot dynamics is surface-dependent, or that the aerodynamics governing a quadrotor are affected by nearby surfaces. Hence, even if enough data is eventually collected and $\tilde f$ is fit such that $d_{W_1}(\tilde{f}(X_t), \allowbreak f(X_t)) \to 0, \forall X_t$ on a training/practice environment (e.g. simulation or lab setting), this approximation quality does not generally translate to inference. Let us make this relationship explicit by writing $Z_t \sim \mathcal P_{train}$ during the synthesis of $\tilde f$ (e.g. by fitting a neural-network, performing system identification, etc.) and $Z_t \sim \mathcal P_{test}$ during deployment. We tackle the general case where $\mathcal P_{train}$ and $\mathcal P_{test}$ are both unknown, and the relationship between them is also unknown, allowing for the introduction of arbitrary epistemic uncertainty to $\tilde f$. \\ \indent \looseness-1
We consider dynamics approximations of the form $\tilde f \colon \mathcal N_{\dim(S)} \times \mathcal U \mapsto \mathcal N_{\dim(S)}$ where the input and output state distributions are MVNs with positive-definite covariance. While this enforces unimodality of our predictions, the ground truth system dynamics $f$ are still free to be non-Gaussian. We do not require any priors about which regions of the inference state-action space are better captured by our model's structure, and instead make the following assumption: 
 \begin{assumption} \label{assump:dataset}
     We have access to a calibration data set of transitions\break
     ${D_{cal} = \{Z^{(i)}\}_{i=1}^{\lvert D_{cal}\rvert}}$, where $Z^{(i)} \sim \mathcal P_{test}$. 
\end{assumption}

We seek to provide probabilistic guarantees for any finite $\lvert D_{cal}\rvert$, but it is expected that a larger dataset can improve the inference performance. As transition data is effectively self-labeled and most robots already track $(s_t, u_t)$ at a high rate (e.g. for localization, planning, etc), obtaining $D_{cal}$ should not induce significant overhead and is reasonable for deployed systems. As in previous CP works \cite{dixit23a,lindemann2023safe,sun2022copula}, we implicitly assume access to the ground-truth transitions making up the calibration dataset. Thus, a formal consideration of the effect state-estimation uncertainty is left for future work, although small amounts of noise should not significantly affect the required calibration factors. \\ \indent
The prediction of the system's state should account for both epistemic and aleatoric uncertainty. We aim to use the predicted uncertainty to control a system of unknown stochastic dynamics towards a goal region, while avoiding unsafe states with provable non-asymptotic safety guarantees, using a model that can be arbitrarily wrong in the uncharacterized deployment domain. We formally introduce the problem we study below. 
\vspace{-1.5em}
\subsubsection{Problem.} Given an approximation $\tilde f$ of the system's \textit{unknown} stochastic dynamics $f$, a goal region $\mathcal G \subseteq \mathcal S$, a safe set $\mathscr C \subseteq \mathcal S$, a calibration dataset $D_{cal}$ and an acceptable failure-rate $\alpha \in (0,1)$, we aim to recursively solve the following stochastic optimization problem with planning horizon $H \in \mathbb N$:
\vspace{-0.6em}
\begin{subequations}
\begin{alignat}{3}
&\!\min_{(u_{t},\ldots,u_{t+H-1})}        &\qquad& J(s, u, \mathcal G)\label{eq:optCost}\\
&\text{subject to} &      & Y_{\tau} \sim f(\bar{X}_\tau),\quad & \forall\tau\in\{t, \ldots, t+H-1\}\label{eq:optDynamics}\\
&               &  & \mathbb P(Y_{\tau} \in \mathscr C) \ge (1-\alpha),\quad & \forall\tau \in \{t,\ldots, t+H-1\}\label{eq:optSafety} \\
&               &  & u_{\tau} \in \mathcal U, s_{\tau+1} \in \mathcal S, \quad &\forall\tau\in\{t, \ldots, t+H-1\} \label{eq:optAdmissible}
\end{alignat}
\end{subequations}

\indent\textit{Dynamics, \eqref{eq:optDynamics}}: The real system evolves following the unknown stochastic dynamics $f$. We only have access to an approximation $\tilde f$ and the inference transitions in $D_{cal}$. \newline
\indent\textit{Safety, \eqref{eq:optSafety}}: Our trajectory should remain probabilistically safe, which we define as requiring it to lie within a safe set $\mathscr C$ with at least a user-specified probability. This is difficult to guarantee in general for any $f$ and $\tilde f$, since our approximation can be arbitrarily wrong in deployment conditions.
\newline
\indent\textit{State and Control Admissibility, \eqref{eq:optAdmissible}}: Both the control inputs and the states must belong to pre-defined sets. \newline
\indent\textit{Objective Function, \eqref{eq:optCost}}: Additionally, we aim to achieve optimality relative to an objective that incorporates $\mathcal G$ along with the state and control sequences, $s := (s_{t+1},\ldots, s_{t+H})$ and $u := (u_{t},\ldots,u_{t+H-1})$. For example, $J$ might minimize the expectation of a distance metric to $\mathcal G$, control effort, or epistemic uncertainty along the state-control sequences.

Practically, the safety condition \eqref{eq:optSafety} can be conservatively re-written as \vspace{-3pt}
\begin{equation}\label{eq:alternative_safety}
    \mathbb P (Y_{\tau} \in \mathcal C(\bar{X}_{\tau})) \ge (1-\alpha) \quad \land\quad \mathcal C(\bar{X}_{\tau}) \subseteq \mathscr C, \quad \forall\tau \in \{t,\ldots, t+H-1\}
    \vspace{-3pt}
\end{equation}
where $\mathcal C\subseteq \mathcal Y$ is a valid prediction region of the future state. This is stronger than condition \eqref{eq:optSafety} since states outside $\mathcal C(\bar{X})$ are not necessarily unsafe. Constructing $\mathcal C$ non-asymptotically is not trivial, since we must account for epistemic uncertainty using only  $D_{cal}$. Additionally, $\mathcal C$ should accurately reflect prediction difficulty (be larger for inputs with greater associated predictive uncertainty) and be as tight as possible. Before presenting our method, we briefly introduce conformal prediction which we will use to construct $\mathcal C$.

\vspace{-10pt}
\section{Split Conformal Prediction}\label{sec:CP}
\label{sec:splitCP}
\vspace{-6pt}
We focus on Split Conformal Prediction (SplitCP), a special case of the original Full Conformal Prediction method \cite{vovk2005algorithmic}, which is more computationally efficient as it only requires fitting the predictor once. We summarize this method below. We recommend the following tutorials \cite{angelopoulos2021gentle,shafer2008tutorial} for a deeper overview of the conformal prediction literature.

For simplicity, we will describe the SplitCP method for a function that maps from point-inputs in $\mathcal X$ to labels in $\mathcal Y$ (we will then discuss how to use the same process for input distributions from $\bar{\mathcal X}$). Let $\hat f \colon \mathcal X \mapsto \mathcal H$ be an arbitrary predictor mapping features to a prediction space $\mathcal H$. While typically in literature $\mathcal H = \mathcal Y$, this is not strictly required (we use $\mathcal H = \mathcal P_{Y\mid X}$). Given a new inference tuple $Z^{(new)}$ which is exchangeable\footnote{The random vector $(Z^{(1)}, \ldots, Z^{\lvert D_{cal}\rvert}, Z^{(new)})$ is \textit{exchangeable} if its joint probability is invariant to permutations of its elements (i.e. they are equally likely to appear in any ordering). This is a weaker requirement than iid, as iid $\implies$ exchangeability.} with $D_{cal}$, SplitCP allows us to generate a feature-dependent prediction region  $\mathcal C\subseteq \mathcal Y$ such that \vspace{-6pt}
\begin{equation}\label{eq:CP_marginal_validty}
    \mathbb P(Y^{(new)} \in \mathcal C(X^{(new)})) \ge 1-\alpha
    \vspace{-6pt}
\end{equation} 
where $X^{(new)}$ is the observed feature and  $Y^{(new)}$ its corresponding unseen label. Thus, given a new inference state-action pair $(s_t, u_t)$, a $\mathcal C$ satisfying equation \eqref{eq:CP_marginal_validty} will contain the \textit{true} next state $(s_{t+1})$ with (at least) a user-determined likelihood. The probability is taken over the randomness in $Z := (Z^{(1)}, \ldots, Z^{\lvert D_{cal}\rvert},\allowbreak Z^{(new)})$, and hence this property is termed \textit{marginal coverage}. Let us define a scalar-valued non-conformity measure $r \colon \mathcal H \times \mathcal Y \mapsto \mathbb R$ encoding the disagreement between the model predictions and the true label such that higher values correspond to a greater disparity. By applying this measure over the calibration tuples we can obtain a set of residuals $R^{(i)}$ as follows \vspace{-6pt}
\begin{equation}
    R^{(i)} := r\left(\hat f(X^{(i)}), Y^{(i)}\right), \quad \forall i \in \{1, \ldots, \lvert D_{cal} \rvert \}
    \vspace{-6pt}
\end{equation}
and define $R := \{R^{(i)}\}_{i=1}^{\lvert D_{cal}\rvert} \cup \{R^{(new)}\}$. Until we observe $Y^{(new)}$ we cannot determine $R^{(new)}$, so we conservatively define $R_\infty:= (R\setminus \{R^{(new)}\})  \cup \{\infty \}$. Let $F_{R_\infty}$ be the empirical (weighted) distribution of the residuals \vspace{-6pt}
\begin{equation} \label{eq:empirical_residual_F}
    F_{R_\infty}(x) = w^{(new)}\mathbbm 1_{\{\infty = x\}} + \sum_{i=1}^{\lvert D_{cal}\rvert}  w^{(i)}\mathbbm 1_{\{ R^{(i)}\le x \}}
    \vspace{-6pt}
\end{equation}
and $\hat q := \mathcal Q \left(1-\alpha; F_{R_\infty} \right)=\inf \{x\in \mathbb R \colon F_{R_\infty}(x) \ge (1-\alpha)\} $ its $(1-\alpha)$-quantile. In SplitCP the weights $w$ are all equal and sum to one. Since $\hat f$, $r$ and $w$ are fixed before
determining $\{R^{(i)}\}_{i=1}^{\lvert D_{cal}\rvert}$, $R$ is exchangeable when conditioned on them. Thus, by a rank statistics argument \cite{tibshirani2019conformal}, the rank of $R^{(new)}$ is uniformly distributed over the indices of $R$ and we can determine that \vspace{-6pt}
\begin{equation}
    \mathbb P(R^{(new)} \le \hat q) \ge (1-\alpha)
    \vspace{-6pt}
\end{equation}
Consequently, the following set $\mathcal C$ is marginally valid: \vspace{-6pt}
\begin{equation}
    \mathcal C(X^{(new)}) := \{ y \in \mathcal Y \colon r(\hat f(X^{(new)}), y) \le \hat q\}
    \vspace{-6pt}
\end{equation} 
While $\mathcal C$ is dependent on $\alpha$, since $\alpha_1 \ge \alpha_2 \implies C_{\alpha_1} \subseteq \mathcal C_{\alpha_2}$, we omit this and regard $\alpha$ as a fixed hyperparameter.
Note that no assumptions were made about the structure of $\hat f$, $\lvert D_{cal}\rvert$ or the distributions $\mathcal P_{train}$ and $\mathcal P_{test}$ (hence why SplitCP is called \textit{distribution-free}). While the choice of $r$ and quality of $\hat f$ do not affect SplitCP's validity, they do impact its adaptiveness and sharpness. 

\vspace{-8pt}
\subsection{Beyond Marginal Coverage}
\vspace{-6pt}
Marginal coverage alone has limited utility for robotic systems. First, it can be trivially achieved by $\mathcal C(X^{(new)}) = \mathcal Y, \forall X^{(new)}$ which is not \textit{sharp}. Second, it does not capture how the predictor's epistemic uncertainty varies across the state-action space (\textit{adaptiveness}), since it depends on a feature-independent $\hat q$, and thus can not help us disambiguate between plans traversing regions of high and low epistemic uncertainty. Finally, it can severely misscover sub-sections of the feature space that are underrepresented in $D_{cal}$. To address these properties of interest for safe motion planning, let us consider stronger guarantees. Ideally, our conformal predictor would provide \textit{(object-)conditional coverage} \cite{pmlr-v25-vovk12}, i.e. satisfy
\begin{equation} \label{eq:cond-validity}
    \mathbb P(Y^{(new)} \in \mathcal C(X^{(new)}) \mid X^{(new)}=x) \ge 1-\alpha, \quad \forall x \in \mathcal X
\end{equation}
so that we can provide probabilistic guarantees for any inference state-action tuple.  Unfortunately, exact conditional validity has been proven to be impossible for a finite number of samples \cite{lei_condImpossible2014}, for all but the trivial predictor $\mathcal C = \mathcal Y$. Our algorithm approximates \textit{asymptotic conditional coverage} under mild assumptions (see \S\ref{sec:local_adaptivity}). Additionally, we achieve \textit{local coverage} \cite{lei_condImpossible2014} with \textit{finite samples}, i.e. satisfying
\begin{equation}\label{eq:loc-validity}
    \mathbb P(Y^{(new)} \in \mathcal C(X^{(new)}) \mid X^{(new)}\in \mathcal X_k) \ge 1-\alpha, \quad \forall k\in \{ 1, \ldots, K\}
\end{equation}
where $\cup_{k=1}^K \mathcal X_k= \mathcal X$ is a finite feature-space partition that we construct automatically from $D_{cal}$. This compromise between conditional and marginal validity enables us to obtain a different $\hat{q}_k$ per partition. For simplicity, we presented SplitCP for point-inputs above. However, our dynamics predictor $\tilde{f}$ takes a distribution over the state as input. We can use the same process as above for distributions by simply concatenating the parameters of the distribution (here the mean and covariance of a MVN) into a vector and using this in place of the point-input. 

\vspace{-10pt}
\section{LUCCa: Local Uncertainty Conformal Calibration}\vspace{-6pt}\looseness-1
Given access to approximate dynamics $\tilde f$, and a finite calibration set of state transitions, we will show how to calibrate $\tilde f$'s estimates to be probabilistically valid for the epistemic uncertainty arising in inference.  We then chain prediction regions of the calibrated model in an MPC scheme which, under additional assumptions about the approximate dynamics and controller (see Appendix \ref{app:proof_Multi_linear}), devises provably safe robot trajectories. Figure \ref{fig:method} showcases the overall method and Algorithm \ref{alg:LUCA} details its steps.

\begin{figure}
\includegraphics[width=\textwidth]{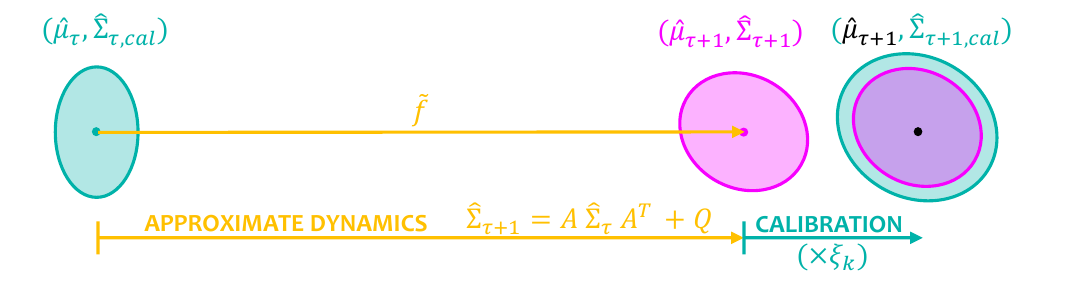}
\caption{Schematic of Local Uncertainty Conformal Calibration (LUCCa). Starting from a calibrated prediction region $\hat{\mathcal N}_{\tau, cal}$, we propagate the state uncertainty by composing an approximate dynamics model $\tilde f$ outputting predictive MVNs of the future state with local conformal calibration. We consider approximate dynamics where the input and output distributions are MVNs of the state, but otherwise do not restrict the structure of $\tilde f$. In our experiments we use an analytical method for propagating a linear $\tilde f$.} \label{fig:method}
\end{figure}

\begin{algorithm} \label{alg:LUCA}
\caption{Local Uncertainty Conformal Calibration (LUCCa)}
\SetInd{0.25em}{1em}
\SetAlgoLined
\DontPrintSemicolon

\KwIn{$\tilde{f}, g, D_{cal}, \alpha, J$}
\For(\tcp*[f]{Collect Calibration Residuals (Offline)})
{$\{X^{(i)},Y^{(i)}\} \in D_{cal}$}
{$\bar{X}^{(i)}\leftarrow g(X^{(i)})$  \tcp*[f]{Map $D_{cal}$ to $\bar{D}_{cal}$, using Eq \eqref{eq:state2MVNmap}}
\\
    $R^{(i)}\leftarrow d_M(\tilde{f}(\bar{X}^{(i)}),Y^{(i)})$ \tcp*[f]{Calculate Calibration Residuals, Eq \eqref{eq:nonConformityScore}}
}
 \texttt{DTree} $\leftarrow$ \texttt{LOCART}($R, D_{cal}, \alpha$) \tcp*[f]{Fit Decision Tree (Once Offline), $\S$\ref{sec:local_adaptivity}}
    
\While(\tcp*[f]{MPC Planning (Online)}){$s_t \notin \mathcal G \land s_t\in \mathscr C$}{\tcp*[l]{Obtain optimal action sequence from MPPI samples using Alg \ref{alg:rollout}}
{$u^* \leftarrow$ \texttt{MPPI}($s_t$, \texttt{CALIBRATED\_ROLLOUT}(\texttt{DTree}$, \tilde f$)$, J$)} \\
$u^*_{t+1} \leftarrow u^*[0]$ \tcp*[f]{Execute 1st action from optimal control trajectory}
}
\end{algorithm}

\begin{algorithm} \label{alg:rollout}
\caption{\texttt{CALIBRATED\_ROLLOUT}\\ \text{(Rollout Approximate Dynamics with Local Conformal Calibration)}}
\SetInd{0.25em}{1em}
\SetAlgoLined
\DontPrintSemicolon

\KwIn{\textnormal{\texttt{DTree}}$, \tilde f$}
        {\eIf(\tcp*[f]{First planning step}){$ \tau = t $} 
            { $\hat{\mathcal N}_\tau \leftarrow \Sigma_0$ \tcp*[f]{Set starting state uncertainty, see Eq \eqref{eq:state2MVNmap}}}
            (\tcp*[f]{Subsequent planning steps}){$ \hat{\mathcal N}_{\tau} \leftarrow \hat{\mathcal N}_{\tau-1, cal}$ \tcp*[f]{Set calibrated uncertainty from previous step}}

        }
        $ \hat{\mathcal N}_{\tau+1} \leftarrow \tilde{f}(\hat{\mathcal N}_{\tau},  u^{(sample)})$ \tcp*[f]{Propagate $\hat {\mathcal N}_\tau$ with MPPI-sampled action} \\
        $\xi_k \leftarrow$ \texttt{DTree}$(\mathbb E(\hat{\mathcal N}_{\tau}), u^{(sample)})$ \tcp*[f]{Query DTree to obtain scaling factor} \\
       $ \hat{\mathcal N}_{\tau+1, cal} \leftarrow \xi_k \hat{\mathcal N}_{\tau + 1}$  \tcp*[f]{Calibrate Approximate Uncertainty, Thm \ref{thm:cal1step}} \\
        \textbf{return} $ \hat{\mathcal N}_{\tau+1, cal} $ 
\end{algorithm}

\vspace{-8pt}
\subsection{Approximate Uncertainty Propagation}
\vspace{-6pt}
We start by showing how to construct predictive safe regions for the ideal case, where $\tilde{f}$ exactly matches $f$, and then show how to use conformal prediction to generalize to an arbitrary discrepancy between $\tilde{f}$ and $f$. 

\begin{lemma}[Theorem 4.5 from \cite{simar19}] \label{lemma:maha-transf}
Let $ Y^{(i)} \sim \mathcal N(\hat \mu_{\tau}, \hat \Sigma_{\tau})$ and its Mahalanobis transformation be $ W^{(i)} =\hat{\Sigma}_{\tau}^{-1/2} ( Y^{(i)}-\hat \mu_{\tau})$.  Then $ W^{(i)} \sim \mathcal N (0, I)$,  where $I$ is an identity matrix of appropriate dimensions.
\end{lemma}

\begin{definition}[Chi-squared Distribution]
Let $ W_j^{(i)} \sim \mathcal N(0,1)$ for \break$j=1,\ldots, \dim({\mathcal S})$. Then the random variable $\sum_{j=1}^{\dim({\mathcal S})}  W_j^{{(i)}^2}$ follows a chi-squared distribution with $\dim({\mathcal S})$ degrees of freedom, denoted $\chi^2_{\dim({\mathcal S})}$.
\end{definition}

We can then show that the squared Mahalanobis distance of the true next state from the predictive distribution follows a $\chi^2$ distribution:

\begin{theorem}
   Let $d_M( Y^{(i)};\hat{\mathcal N}_\tau) := \sqrt{( Y^{(i)}-\hat \mu_\tau)^\top \hat{\Sigma}^{-1}_\tau( Y^{(i)}-\hat{\mu}_\tau)}$ be the Mahalanobis distance from $ Y^{(i)}$ to $\hat{\mathcal N}_\tau$. Then $d_M^2( Y^{(i)};\hat{\mathcal N}_\tau)\sim \chi_{\dim(\mathcal S)}^2$.
\end{theorem}
\begin{proof}
    Since $d_M^2( Y^{(i)};\hat{\mathcal N}_\tau)= W^\top  W=\sum_{j=1}^{\dim({\mathcal S})}  W_j^{{(i)}^2}$, then $d_M^2\sim \chi^2_{\dim(\mathcal S)}$.
\end{proof}
 The following region would then contain the next state $(1-\alpha)\%$ of the time

\begin{lemma}[Result 4.7 from \cite{johnson88}]
\label{lemma:prob_mass_covariance}
Let $\hat{\Sigma}_\tau$ be positive definite and $\chi^2_{\dim(\mathcal S),\alpha} := \mathcal Q(1-\alpha; \chi^2_{\dim(\mathcal S)})$ be the $(1-\alpha)$-quantile of the $\chi^2_{\dim(\mathcal S)}$ distribution. Then the solid hyperellipsoid 
\vspace{-8pt}
\begin{equation*}
    \varepsilon := \{ y \in \mathcal Y :d_M^2( y;\hat{\mathcal N}_\tau)\le \chi^2_{\dim(\mathcal S),\alpha} \}
    \vspace{-6pt}
\end{equation*}
contains ${100(1-\alpha)\%}$ of the probability mass of $\hat{\mathcal N}_\tau$, i.e. $\mathbb P (\hat Y^{(i)} \in \varepsilon ) = 1-\alpha$.
\end{lemma}
However, in practice $\tilde f$ can be arbitrarily inaccurate. We now show how to transform the output of $\tilde f$ so that it remains probabilistically valid.

\vspace{-8pt}
\subsection{Calibrating Uncertainty Estimates} \label{sec:calibrate}
\vspace{-6pt}
To provide performance guarantees without placing assumptions on the inference dynamics, we will leverage SplitCP. Let us define a deterministic map $g\colon \mathcal X \mapsto \mathcal N_{\dim(\mathcal S)} \times \mathcal U$ transforming state-action tuples as follows \vspace{-6pt}
\begin{equation}\label{eq:state2MVNmap}
    \bar{X}^{(i)} = g(X^{(i)})=\left(\mathcal N_{\dim(\mathcal S)}(s_t^{(i)}, \Sigma_0), u_t^{(i)}\right)
    \vspace{-6pt}
\end{equation}\looseness-1 
where $\Sigma_0$ is a small user-selected initial uncertainty level. We can then construct $\bar{D}_{cal}$ by replacing all $X^{(i)}$ by $\bar{X}^{(i)}$. It follows that for a new inference state-action pair $X^{(new)} \sim \mathcal P_{test}$ which is exchangeable with $D_{cal}$, the tuple $\bar{X}^{(new)}=g(X^{(new)})$ is exchangeable with $\bar{D}_{cal}$ as $g$ is fixed. We define the non-conformity metric $r$ as the Mahalanobis distance between the approximate uncertainty estimate obtained from propagating $\bar{X}^{(i)}$ and the true next state, i.e.
\begin{equation} \label{eq:nonConformityScore}
    R^{(i)} = r(\tilde f(\bar{X}^{(i)}), Y^{(i)}) := d_M(\tilde f(\bar{X}^{(i)}), Y^{(i)})
\end{equation}
Observe that, for a fixed $Y^{(i)}$ and $X^{(i)}$, $R^{(i)}$ will decrease as the the growth in uncertainty due to $\tilde f$ increases. We can now define a set of residuals $R$, $R_\infty$ and $\hat q$ as in $\S$ \ref{sec:splitCP}. Let us use this to calibrate the approximate uncertainty propagation in the first step planning step:

\begin{theorem}\label{thm:cal1step}
Let $\hat{\mathcal N}_{t+1} =\tilde f(\bar{X}^{(new)})$ be the predictive distribution for a new test input $X^{(new)}$ and $\xi :=\nicefrac{\hat{q}^2}{\chi^2_{\dim(\mathcal S),\alpha}}$. Then the solid hyperellipsoid containing ${(1-\alpha)\%}$ of the probability mass of $\hat{\mathcal N}_{t+1,cal}(\hat{\mu}_{t+1}, \hat{\Sigma}_{t+1,cal})$, where $\hat{\Sigma}_{t+1,cal} =\xi \hat{\Sigma}_{t+1}$, contains $Y^{(new)}$ with probability greater than or equal to $(1-\alpha)$.
\end{theorem}
\begin{proof}
From Lemma \ref{lemma:prob_mass_covariance}, the $(1-\alpha)$ confidence region of {$\hat{\mathcal N}_{t+1,cal}(\hat \mu_{t+1}, \hat{\Sigma}_{t+1, cal})$ is}
\vspace{-16pt}
\begin{align*}
\varepsilon= &\ \{y\in \mathcal Y \colon (y-\hat \mu_{t+1})^\top(\hat{\Sigma}_{t+1,cal})^{-1}(y-\hat \mu_{t+1}) \le \chi^2_{\dim(\mathcal S),\alpha} \} \\
=&\ \{y\in \mathcal Y \colon (y-\hat \mu_{t+1})^\top(\xi \hat \Sigma_{t+1} )^{-1}(y-\hat \mu_{t+1}) \leq \chi^2_{\dim(\mathcal S),\alpha} \} \\
=&\ \{y\in \mathcal Y \colon (y-\hat \mu_{t+1})^\top\hat \Sigma^{-1}_{t+1} (y-\hat \mu_{t+1}) \leq \hat q^2 \} \\
=&\ \{y\in \mathcal Y \colon d_M^2(\hat{\mathcal N}_{t+1},y) \leq \hat q^2 \}
\vspace{-12pt}
\end{align*}
Additionally, by direct application of SplitCP we have that \break${\mathbb P(r(\hat{\mathcal N}_{t+1}, Y^{(new)}) \le \hat q) \ge (1-\alpha)}$. Thus $\mathbb P(d_M^2(\hat{\mathcal N}_{t+1}, Y^{(new)}) \le \hat q^2) \ge (1-\alpha)$ and so $\mathbb P(Y^{(new)} \in \epsilon) \ge 1-\alpha$. 
\end{proof}
Hence, scaling the approximate covariance matrix obtained from $\tilde f$ by the conformal scaling factor $\xi$ allows us to calibrate its uncertainty so that its $(1-\alpha)$ confidence region now contains at least $(1-\alpha)\%$ of the future states. This scaling factor acts as an interpretable measure of the model error of $\tilde f$. While the above enables the construction of a provably safe first prediction step, greater care must be taken when sequencing predictions over multiple steps. After the first prediction, the MVN passed as input to $\tilde f$ is not necessarily exchangeable with $\bar{D}_{cal}$ since it results from a different procedure, thus we do not provide a hard guarantee that the calibrated plans satisfy condition \eqref{eq:optSafety} for arbitrary approximate dynamics $\tilde f$. For guarantees about the multi-step case see Appendix \ref{app:proof_Multi_linear}. We empirically validate multi-step coverage in $\S 6$.

\vspace{-10pt}
\subsection{Local Calibration}\label{sec:local_adaptivity}
\vspace{-4pt}
While the procedure above guarantees marginal coverage, it uses the same scaling factor $\xi$ for all input state distributions and actions. To capture how the model error varies across the state-action space, we use the LOCART procedure first introduced in \cite{dtree_brazil}. First, $\bar{D}_{cal}$ is randomly split with a user-defined ratio into two disjoint sets $\bar{D}_{cal,part}$ and $\bar{D}_{cal,scale}$. The data in $\bar{D}_{cal,part}$ is used to fit a decision tree regressor that takes as inputs the features $\bar{X}^{(i)}$ and returns an estimate of the corresponding residual $R^{(i)} = d_M(\tilde f(\bar{X}^{(i)}), {Y}^{(i)})$. To accomplish this, the decision tree partitions the feature space into regions $\mathcal X_k$ of small residual variance using the CART algorithm \cite{breiman2017classification}. Effectively, this allows us to disambiguate regions of high and low model accuracy. \\ \indent
The second set $\bar{D}_{cal,scale}$ is used to perform conformal prediction on each individual feature space partition leading to $k$ different scaling factors $\xi_k$ (one per partition). Let $\mathcal T: \bar{\mathcal X} \mapsto \mathcal X_k$ map a state-action pair to its corresponding partition as generated from $\bar{D}_{cal,part}$. This is achieved by passing $\bar X^{(i)}\in D_{cal,scale}$ through the decision tree and observing the index $k$ of its leaf node. Practically, one can interpret the decision tree as assigning a weight of $1$ to calibration tuples whose features lie on the same leaf node as $\bar{X}^{(new)}$ and $0$ otherwise. That is, we now consider the empirical weighted residual distribution of Equation \eqref{eq:empirical_residual_F} with $w^{(i)}=\mathbbm 1_{\{ \mathcal T(\bar X^{(i)}) = \mathcal T(\bar X^{(new)})\}}, \forall i \in \{1, \ldots, \lvert D_{cal}\rvert, (new)\}$ and apply conformal prediction directly to it (after normalizing the weights such that they sum to one). Since the feature space partition is fixed before inference, $\bar{X}^{(new)}$ is still exchangeable with $\bar{D}_{cal,scale}$ and, as proved in \cite{dtree_brazil}, this leads to \textit{finite-sample local coverage guarantees}. Under additional assumptions that $\mathcal X_k$ contain "enough elements" and are "sufficiently small", \cite{dtree_brazil} shows that as $\lvert \bar{D}_{cal}\rvert \to \infty$ this procedure also achieves asymptotic conditional coverage. \\ \indent
We call the process of partitioning the state-action space and assigning a $\xi_k$ to each leaf node "fitting" the decision tree. During execution, each $\bar{X}^{(new)}$ can be passed through the decision tree to determine its leaf node and corresponding $\xi_k$. Naturally, higher dimensional state-action spaces will require more calibration tuples be to properly covered, so that LOCART can produce informative partitions. 
Also, the effective number of calibration tuples in local conformal calibration is partition-dependent and equal to the number of tuples reaching each partition's leaf node. Hence, the more splits the decision tree has, the more transitions will be needed to adequately cover all partitions.\\ \indent
In our implementation, we only pass the state mean and action as inputs to the decision tree, and thus do not separate between different initial uncertainties. Implicitly, the magnitude of the epistemic uncertainty must also be large enough when compared to the aleatoric uncertainty, to allow for informative feature space partitions. With this local procedure, our calibrated regions will not be undercovered for state-action pairs underrepresented in $D_{cal}$, might be more adaptive, and our planner has access to a scalar $\xi_k$ encapsulating the model's accuracy. As with SplitCP, the procedure above applies directly for one-step prediction no matter the form of $\tilde f$.

\vspace{-10pt}
\section{Experiments}
\vspace{-8pt}
Our experiments aim to validate LUCCa in a scenario where there is significant error in the approximate dynamics. We use a double-integrator system with the following state, action, and nominal dynamics as derived in \cite{cosner2023robust}:\vspace{-6pt}
\begin{equation}
    \label{eq:white_dynamics}
    s_{t+1} = A s_t + B u_t + d_t = \begin{bmatrix} I_{2\times 2} & \Delta t I_{2\times 2} \\ 0_{2\times 2} & I_{2\times 2}
    \end{bmatrix} s_t + \begin{bmatrix}
        \frac{\Delta t^2}{2} I_{2\times 2} \\ 
        \Delta t I_{2\times 2}
    \end{bmatrix} u_t + d_t
    \vspace{-6pt}
\end{equation}
where the state $s_t=[p_{x,t}, p_{y,t}, v_{x,t}, v_{y,t}]^\top$ is composed of 2D positions and velocities, the action $u_t = [a_{x,t}, a_{y,t}]^\top$ of 2D accelerations, $\Delta t=0.05$ sec is the integration step, and $d_t \sim Q=\mathcal N(0_{4\times 1}, 0.1 B B^\top)$ an aleatoric noise term. Note that this system can build significant momentum and thus underestimating the uncertainty and coming too close to obstacles may lead to collision, even if the predictor correctly captures the uncertainty for the next time step (i.e. if no action can counteract the momentum completely within one step). Thus, it is important to accurately predict the uncertainty multiple time steps into the future to avoid collision. We also include regions where the true dynamics abruptly shift, i.e. representing lower friction. Hence, $f$ is hybrid. In these regions, the robot dynamics are given by
\vspace{-4pt}
\begin{equation}\label{eq:yellow_dynamics}
    s_{t+1} = \begin{bmatrix} I_{2\times 2} & 1.3\Delta t I_{2\times 2} \\ 0_{2\times 2} & I_{2\times 2}
    \end{bmatrix} s_t + 1.3 B u_t + d_t
    \vspace{-4pt}
\end{equation}
\looseness-1 We conduct tests in four environments containing regions (white areas of Fig \ref{fig:all_runs}) where the approximate dynamics exactly matches the true dynamics ($\text{Eq \eqref{eq:white_dynamics}}$), and regions (yellow areas in Fig \ref{fig:all_runs}) where the true and approximated dynamics are significantly different, approximate is $\text{Eq \eqref{eq:white_dynamics}}$ and real is $\text{Eq \eqref{eq:yellow_dynamics}}$. An effective uncertainty predictor would predict lower uncertainty in the white regions and higher uncertainty (due to epistemic uncertainty) in the yellow region.
\begin{figure}[t!]
\makebox[\textwidth][c]{\includegraphics[width=\textwidth]{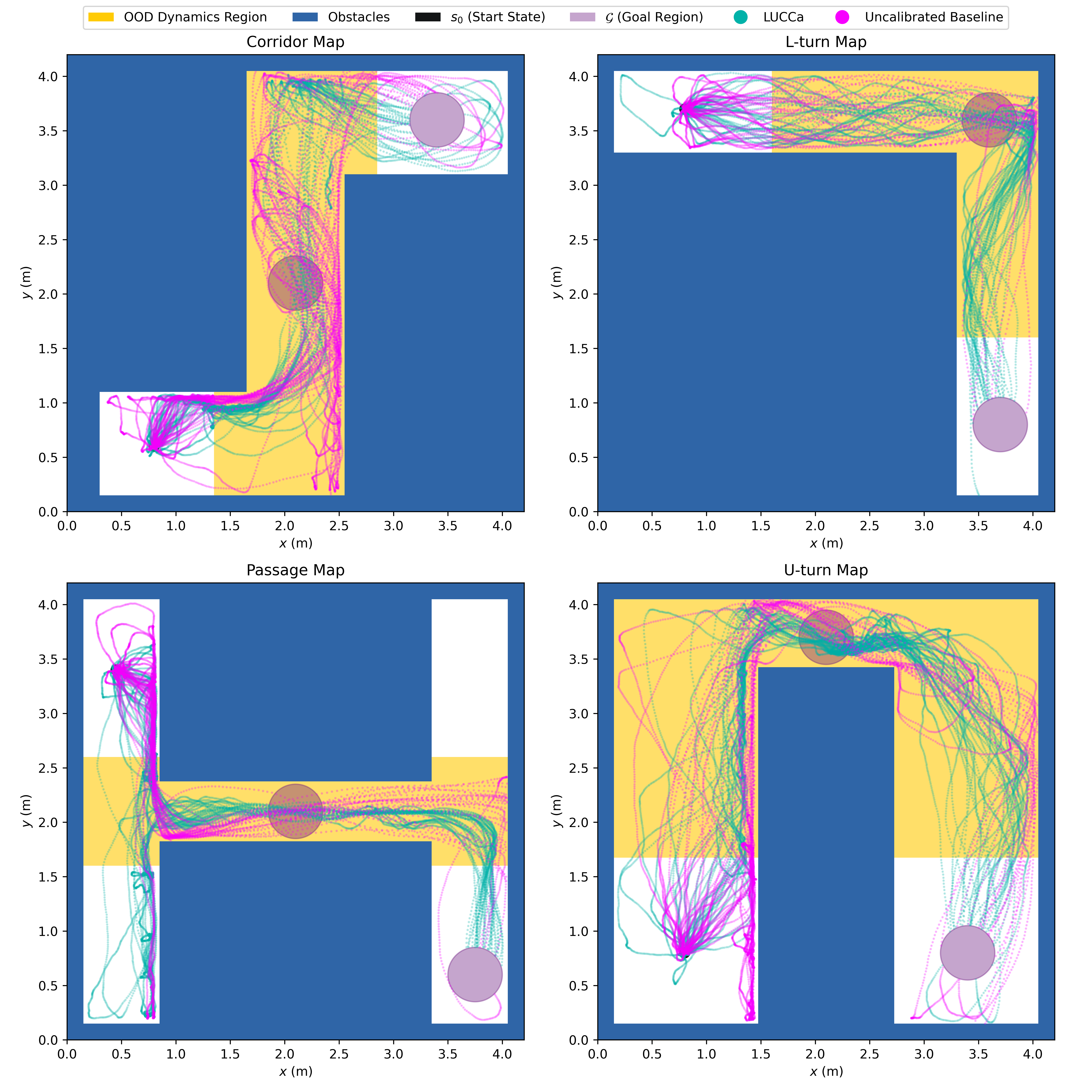}}
\vspace{-12pt}
\caption{Executed trajectories of LUCCa and baseline across four environments. The intermediate sub-goals help simplify the cost landscape. The baseline cannot account for the epistemic uncertainty and gains too much momentum in its pursuit of the sub-goals, making future collisions inevitable. LUCCa accounts for local shifts in model error, making safer turns under both aleatoric and epistemic uncertainty.} \label{fig:all_runs}
\end{figure}
We construct $D_{cal}$ by rolling out state-action pairs on a uniform grid for one step using $f$. To appropriately cover the space, we use a grid where $v_x,v_y\in \{-1, -\frac{1}{2}, 0, \frac{1}{2}, 1\}$, $a_x,a_y\in \{-0.8, -0.4, 0, 0.4, 0.8\}$ and $p_x, p_y \in {\{0, \frac{4.2}{15}, \frac{8.4}{15}, \ldots, 4.2\} \cap \mathscr C}$. This ensures we appropriately cover the entire state-action space (with $\lvert \bar{D}_{cal}\rvert=53'750$ for the Corridor map). In LOCART, we set the random split such that $\lvert \bar{D}_{cal,part} \rvert= \lvert \lceil0.8\bar{D}_{cal}\rceil \rvert$, set the decision tree max depth to $13$ and the minimum number of samples required for a split to $40$. Empirically, if too small a ratio of calibration samples is used for feature space partition, the resulting $\xi_k$ plots might not (due to randomness in decision tree regressor and presence of aleatoric noise) properly match the true changes in model accuracy. For the parameters and $\lvert \bar{D}_{cal}\rvert$ above, calculating all the residuals and fitting the decision tree (creating the state-action space partitions and assigning a conformal scaling factor to each one) takes approximately $0.17$ secs on an Intel i9-12900K. Note this must only be performed once prior to planning (according to Algorithm \ref{alg:LUCA}). In Appendix \ref{app:visualize_xi_k}, we show some $\xi_k$ obtained after fitting LOCART on each environment, and relate these scaling factors to the regions of high/low epistemic uncertainty. \\ \indent
The remainder of this section describes the two types of experiments we conducted. First, we sought to verify our claim that LUCCa produces a set that contains at least $(1-\alpha)$ percent of the outputs of the true dynamics for the first planning step. Additionally,
although the presented system with bounded control inputs and $\dim(\mathcal S)>\dim(\mathcal U)$ is more complex than those for which we provably extended LUCCa's safety guarantees to the multi-step case (Appendix \ref{app:proof_Multi_linear}), we wanted to test whether LUCCa could still provide $(1-\alpha)$ percent safety along each step of its plan. Second, we conducted experiments on an MPC controller that uses LUCCa to plan short-horizon trajectories to reach a goal. For both experiments, we compare to an uncalibrated baseline, i.e. directly using the uncertainty produced by the approximate dynamics, to show the importance of considering epistemic uncertainty. We use $\alpha=0.1$ for all experiments.
\vspace{-6pt}
\subsection{Empirical Coverage}
\begin{wrapfigure}[18]{r}{0.55\textwidth}
    \centering
    \vspace{-42pt}
    \includegraphics[width=0.5\textwidth]{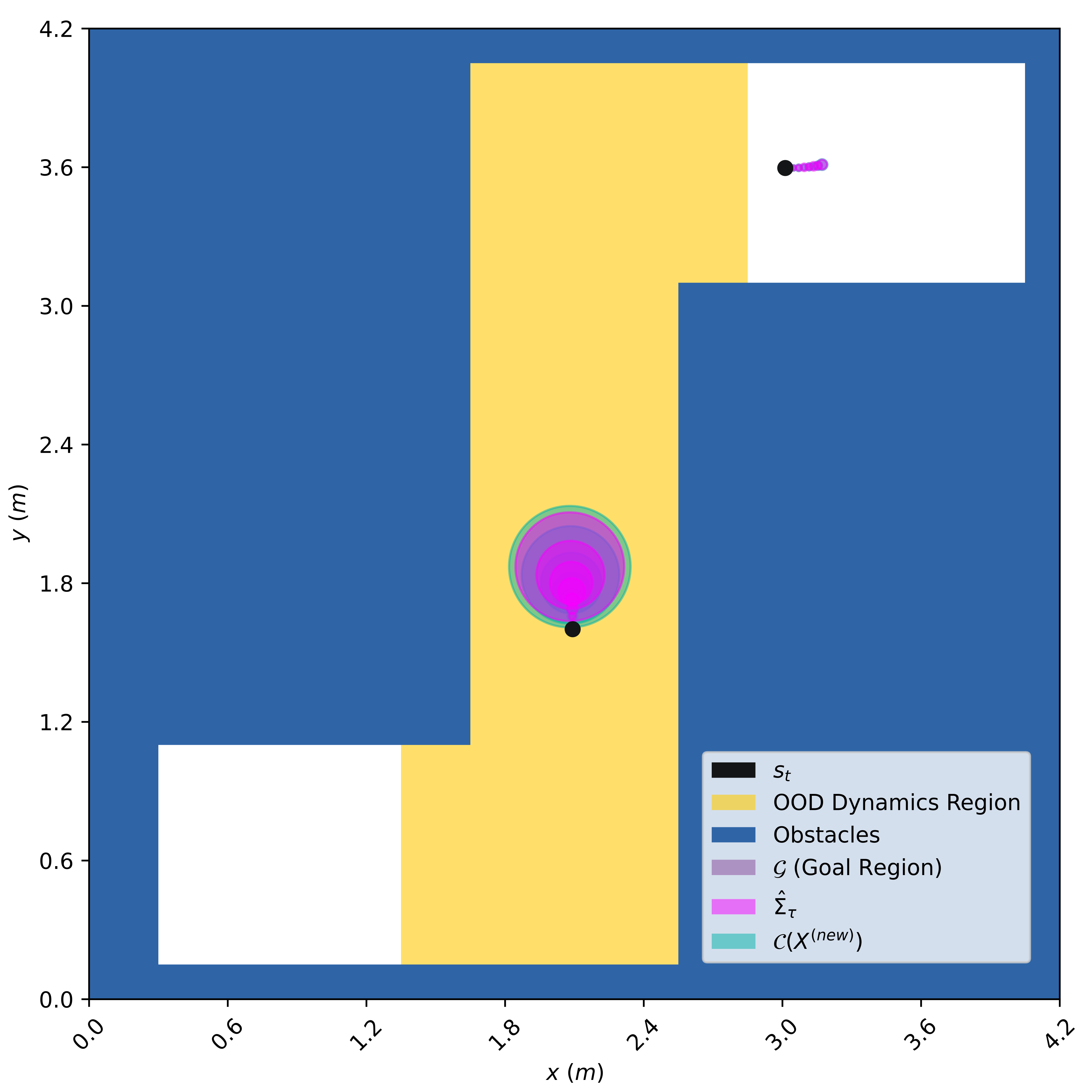}
    \caption{We show two optimal plans starting within the white and yellow regions. These are the starting states $s_0$ for the empirical coverage calculation.}
    \label{fig:trajectory}
    \vspace{-80pt}
\end{wrapfigure}
To verify the central claim of this paper, that in the first planning step LUCCa outputs a set that contains at least $(1-\alpha)$ percent of the outputs of the true dynamics, we conducted an empirical coverage evaluation in the Corridor environment. Starting from two states (one in the white region $s_0=[3.0, 3.6, 0.5, 0.0]^\top$ and one in the yellow $s_0=[2.1, 1.5, 0.0, 0.5]^\top$), in the environment in Fig. \ref{fig:trajectory}, we applied a sequence of constant either accelerating or decelerating controls for the duration we use as our planning horizon. We accelerate and decelerate in the yellow region by $\pm 0.5$ m/s$^2$, and in the white region by $\pm 0.8$ m/s$^2$. We compare the ratio of $2500$ samples that fall within the predicted uncertainty region for each time step for LUCCa vs. the baseline in Fig. \ref{fig:empirical_coverage}. For the baseline, the prediction region at each step can be obtained analytically by following the update rule $\hat{\Sigma}_{\tau + 1} = A \hat{\Sigma}_\tau A^\top + Q$ and applying Lemma \ref{lemma:prob_mass_covariance} to the resulting distribution. 
The figure shows that at least 90\% of samples are within LUCCa’s first-step prediction regions in all cases, as expected, though LUCCa does appear to be more conservative than necessary (especially in later time steps).
It also validates that, despite the tested system being more complex than those for which we proved multi-step safety, LUCCa appears to also provide multi-step local coverage for linear approximate dynamics.
Conversely, the baseline, which does not consider epistemic uncertainty, significantly under-predicted the uncertainty for the state where the dynamics were inaccurate. For both states where $f=\tilde f$, we see that the baseline provides valid prediction regions (around $90\%$ of the samples are contained within the $90\%$ confidence hyperellipsoid) and LUCCa is slightly over-conservative. For the states where there's significant model inaccuracy, LUCCa's regions remain empirically valid for all planning steps while the baseline ones underestimate uncertainty in the true future states (with decaying validity over time). 
\vspace{-10pt}
\begin{figure}[t]
  \makebox[\textwidth][c]{\includegraphics[width=0.9\textwidth]{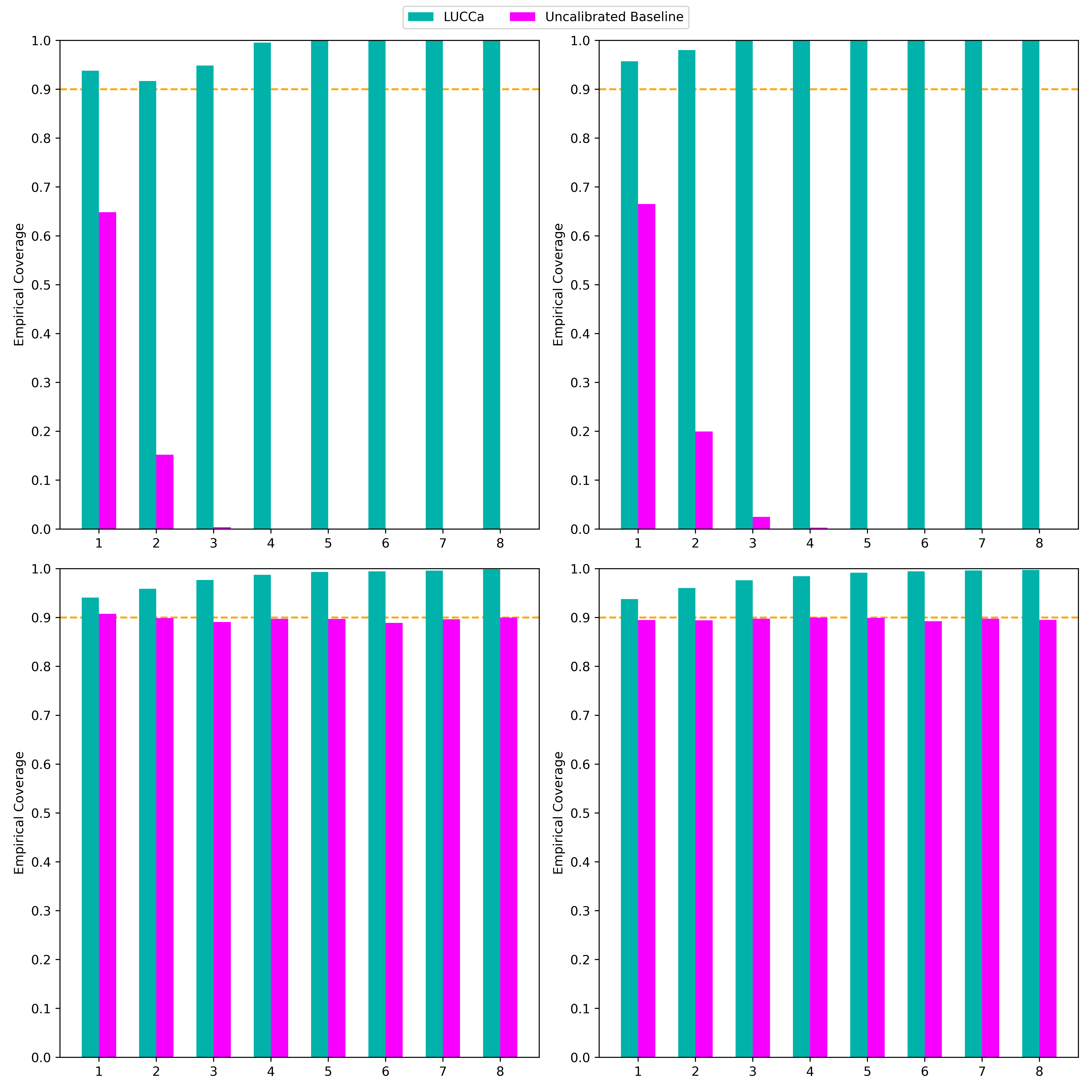}}
\caption{Empirical coverage of LUCCa (teal) and an uncalibrated baseline (pink) in four examples. Top: Decelerating (left) and  Accelerating (right) in a region with epistemic uncertainty. Bottom: Same for a region with no epistemic uncertainty. LUCCa's estimate is conservative -- always above the user-defined 90\% threshold. The uncalibrated baseline significantly under-estimates the uncertainty when $\tilde f$ is inaccurate.} \label{fig:empirical_coverage}
\end{figure}

\subsection{Using LUCCa's Predictions for Motion Planning}

To evaluate the usefulness of LUCCa's uncertainty predictions for motion planning, we conducted an experiment where the double-integrator had to traverse the environments in Fig. \ref{fig:all_runs} to reach a goal. We used MPPI \cite{mppi} to control the robot with $\lambda=1$, $4096$ sampled trajectories, control bounds $u_\tau \in [-0.9, 0.9]\times [-0.9, 0.9]$ and control noise $I_{2\times 2}$. If the horizon $H$ is too small, the planner can act myopically (reaching high-speeds and possibly entering regions of inevitable collision), however if $H$ is too high then the calibration regions can become excessively large making all paths probabilistically unsafe (due to successively scaling by $\xi_k$). $H=8$ empirically achieved a satisfactory balance. We use as objective function \vspace{-6pt}
\begin{equation}\label{eq:MPPI_cost}
    J(s,u,\mathcal G) = 1.15d_E(s_{t+H}, \mathcal G) + \sum_{\tau=t+1}^{t+H-1} 1.1d_E(s_{\tau}, \mathcal G) + 1.8\tr(\hat{\mathcal N}_{\tau, cal}) + 10^5 \mathbbm 1_{col}(\hat{\mathcal N}_{\tau, cal})
    \vspace{-3pt}
\end{equation}
\looseness-1 where $d_E$ is the Euclidean distance to the center of $\mathcal G$, $\tr(\cdot)$ the trace of a matrix and $\mathbbm 1_{col}$ is an indicator function determining collisions between the $(1-\alpha)$ confidence region of $\hat{\mathcal N}_{\tau, cal}$ an the known environmental obstacles (soft safety constraint). If the obstacles can be decomposed into rhomboids, $\mathbbm 1_{col}$ can be simplified to a triangle-circle check after applying a Mahalanobis Transform with $\hat{\mathcal N}_{\tau, cal}$ to the environment. The trace term penalizes states of greater total uncertainty. In practice, we observed this leads to the system slowing down when subject to greater model mismatch. Empirically, normalizing the terminal distance-to-goal cost across the sampled MPPI trajectories led to reaching the goal in fewer steps. Each environment contained a sub-goal used to simplify the cost-landscape. Upon entering the sub-goal region, the controller's goal was updated to be the map's second sub-goal region. We define a successful run as one where the system reaches the second sub-goal without ever colliding with any obstacles. \\ \indent Figure \ref{fig:all_runs} compares the trajectories produced by LUCCa and the baseline across the four environments. Videos of some executions can be found on our companion website \href{https://um-arm-lab.github.io/lucca}{\textbf{\color{umich}https://um-arm-lab.github.io/lucca}}. The results in Table \ref{tab:planning} suggest that using LUCCa's uncertainty estimate improves the success rate in these scenarios by avoiding collision. However, LUCCa only guarantees that $(1-\alpha)$ percent of true states will be collision-free, so it does not provide a hard guarantee that the planned actions will result in collision-free states. We also note that the baseline did reach the goal faster when it didn't collide. This is because its trace cost term is smaller and independent of the system's current position, leading to greater velocities overall. On an Intel i9-12900K processor, the calibration process in LUCCa (passing MPPI sample through the decision tree to determine its leaf node and calibrating the approximate covariances with the corresponding scaling factor) was found to introduce a $0.3$ ms overhead on each planning step (averaged over $91000$ planning steps). While this represents a $32\%$ increase in the planning step duration, the magnitude of the additional computations is neglectable in practice (and could be further mitigated by GPU processing).

\begin{table}
\begin{center}
\caption{Comparison of LUCCa vs baseline in four environment over 30 runs (each).}
\begin{tabular}{ |c|c|c|c|c| } 
\hline
\multirow{2}{*}{Environment} & \multirow{2}{*}{Prediction Method} & \multirow{2}{*}{Collision Rate $\downarrow$} & \multirow{2}{*}{Success Rate $\uparrow$} & Avg. Steps \\ & & & & to Goal (std) $\downarrow$\\
\hline
\multirow{2}{*}{Corridor} & LUCCa & \textbf{0.00} & \textbf{1.00} & \textbf{385.83} (102.58) \\ 
& Uncalibrated Baseline  & 0.30 & 0.70 & 410.14 (137.24) \\
\hline
\multirow{2}{*}{L-Turn} & LUCCa & \textbf{0.03} & \textbf{0.97} & 363.48 (70.00) \\ 
& Uncalibrated Baseline  & 0.80 & 0.20 & \textbf{292.17} (39.12) \\ 
\hline
\multirow{2}{*}{Passage} & LUCCa & \textbf{0.00} & \textbf{1.00} & 505.33 (192.34) \\ 
& Uncalibrated Baseline  & 0.83 & 0.17 & \textbf{367.20} (139.11) \\ 
\hline
\multirow{2}{*}{U-Turn} & LUCCa & \textbf{0.00} & \textbf{1.00} & 636.67 (109.65) \\ 
& Uncalibrated Baseline  & 0.40 & 0.50 & \textbf{350.07} (48.33) \\ 
\hline
\end{tabular}
\label{tab:planning}
\end{center}
\end{table}

\vspace{-10pt}
\section{Conclusion}
\vspace{-6pt}
This paper presented LUCCa, a method based on conformal prediction that accounts for both aleatoric and epistemic uncertainty in the dynamics. The method uses a calibration set of data to locally estimate a scaling factor, which is applied to the output of the approximate dynamics predictor. Our experiments on a double-integrator system corroborate our analysis, showing that LUCCa outputs a set containing at least $(1-\alpha)$ percent of the results of the true dynamics for one-step predictions. Further, we prove and empirically observe that for linear approximate dynamics, under additional controller assumptions, LUCCa can also provide local coverage for the remaining planning steps. We also show that LUCCa's predictions are effective for safely guiding an MPC method to reach a goal.
\vspace{-0.8em}
\begin{credits}
\subsubsection{\ackname}This work was supported in part by the Office of Naval Research Grant N00014-24-1-2036 and NSF grants IIS-2113401 and IIS-2220876.
\vspace{-1em}
\subsubsection{\discintname} The authors have no competing interests to declare that are relevant to the content of this article.
\end{credits}

\bibliographystyle{splncs04}
\bibliography{references}

\newpage
\appendix
\section{Proving LUCCa's Multi-Step Coverage with Linear Dynamics Approximators} \label{app:proof_Multi_linear}

\input{appendixA}

\newpage
\section{Visualizing Conformal Calibration Factors across environments} \label{app:visualize_xi_k}

In Figures \ref{fig:conformal_scaling_factor_plot}, \ref{fig:scaling_factor_uturn}, \ref{fig:scaling_factor_passage} and \ref{fig:scaling_factor_lturn}, we showcase how the local scaling factors $\xi_k$ vary with changes in horizontal velocity $v_x$ across the four environments of Figure \ref{fig:all_runs}. The colormaps were obtained by querying the fitted decision trees with $p_x, p_y \in (0, 4.2)$, $v_y=a_x=a_y=0$ and $v_x \in \{-2, 0, -2\}$. The horizontal velocity increases from left to right. As expected, the regions of higher $\xi_k$ correspond to the yellow region (higher epistemic uncertainty). While we note that $\xi_k$ are different for $v_x=-2$ m/s and $v_x=2$ m/s, this may result from aleatoric disturbances or the inherent randomness of the $\bar{D}_{cal}$ split in $\S$\ref{sec:local_adaptivity} (i.e. which samples were used for partitioning the feature space and which were used for determining the scaling factors).

\begin{figure}[H]
  \makebox[\textwidth][c]{\includegraphics[width=1.2\textwidth]{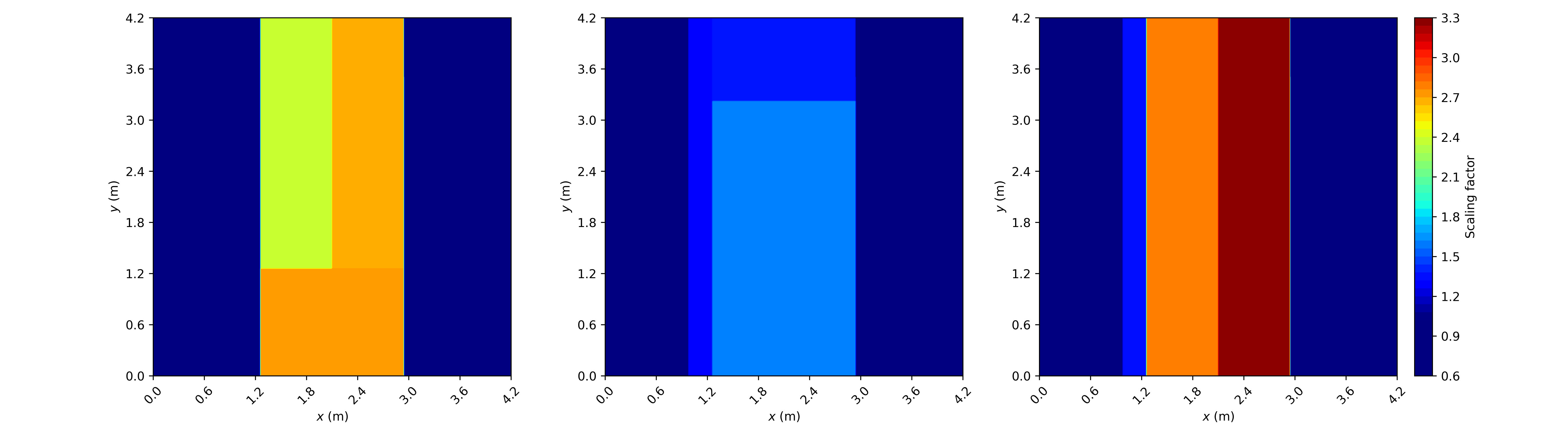}}\vspace{-6pt}
\caption{\looseness-1 Local Conformal Scaling Factor $\xi_k$ over double integrator positions in Corridor environment. Obtained by querying the fitted \texttt{DTree} over all environment positions with $a_x=a_y=v_y=0$ and (from left to right): $v_x=-2$ m/s, $v_x=0$m/s, $v_x=2$ m/s.}  \label{fig:conformal_scaling_factor_plot}
\end{figure}

\begin{figure}[H]
  \makebox[\textwidth][c]{\includegraphics[width=1.2\textwidth]{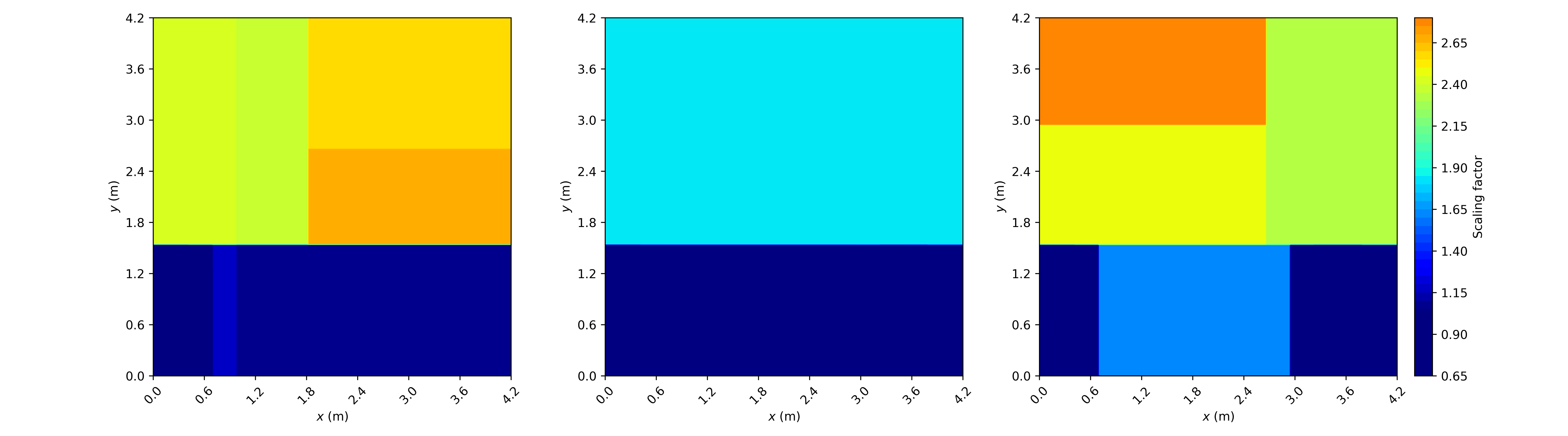}}\vspace{-6pt}
\caption{Local Conformal Scaling Factor $\xi_k$ over double integrator positions in U-Turn environment. Obtained by querying the fitted \texttt{DTree} over all environment positions with $a_x=a_y=v_y=0$ and (from left to right): $v_x=-2$ m/s, $v_x=0$m/s, $v_x=2$ m/s.} \label{fig:scaling_factor_uturn}
\end{figure}
\vspace{-30pt}
\begin{figure}
  \makebox[\textwidth][c]{\includegraphics[width=1.2\textwidth]{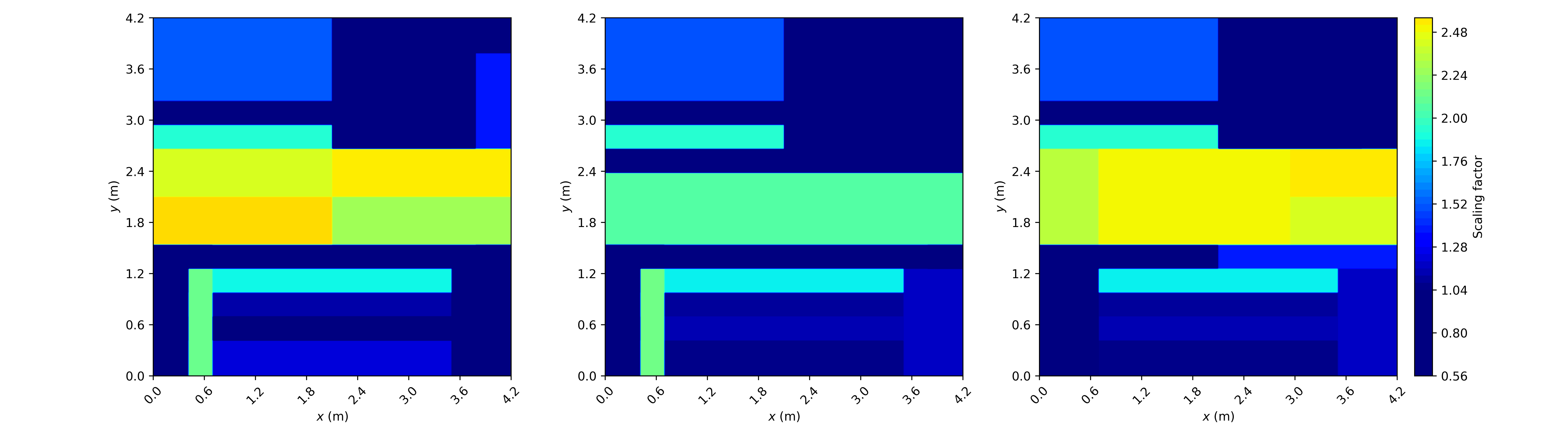}}\vspace{-6pt}
\caption{Local Conformal Scaling Factor $\xi_k$ over double integrator positions in Passage environment. Obtained by querying the fitted \texttt{DTree} over all environment positions with $a_x=a_y=v_y=0$ and (from left to right): $v_x=-2$ m/s, $v_x=0$m/s, $v_x=2$ m/s.} \label{fig:scaling_factor_passage}
\end{figure}
\vspace{-30pt}
\begin{figure}[b]
  \makebox[\textwidth][c]{\includegraphics[width=1.2\textwidth]{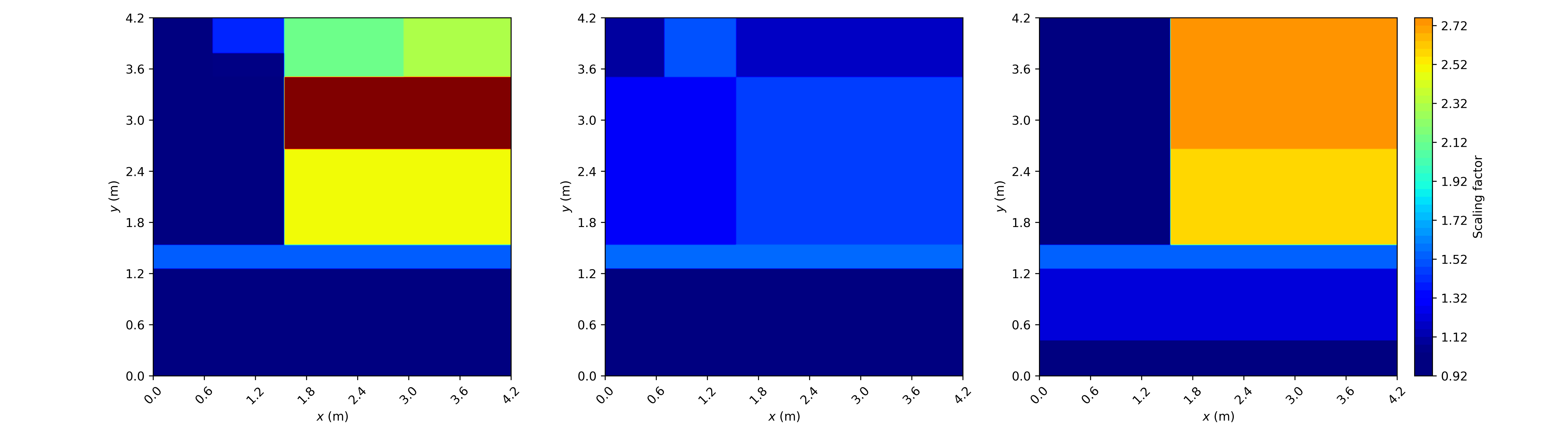}}\vspace{-6pt}
\caption{Local Conformal Scaling Factor $\xi_k$ over double integrator positions in L-Turn environment. Obtained by querying the fitted \texttt{DTree} over all environment positions with $a_x=a_y=v_y=0$ and (from left to right): $v_x=-2$ m/s, $v_x=0$m/s, $v_x=2$ m/s.} \label{fig:scaling_factor_lturn}
\end{figure}

\end{document}

%% file: appendixA.tex
Our proposed algorithm (LUCCa) used MPPI \cite{mppi}, a sampling-based shooting-style MPC scheme, to compute ``optimal'' control input sequences that balance goal progression, obstacle avoidance, and uncertainty reduction - see Equation \eqref{eq:MPPI_cost} and Algorithm \ref{alg:LUCA}. Empirically, MPPI has been shown to be effective for a wide range of approximate dynamics forms (e.g. nonlinear, non-differentiable) and constraint types (e.g. nonlinear). Consequently, the quality of MPPI's output is highly dependent on the structure of the problem at hand, which we kept broad in the main paper, and on the available compute budget (e.g. number of sampled trajectories, planning horizon, range of control disturbances). Therefore, it is difficult to derive general claims about safety resulting from MPPI plans. To make progress, we will make the simplifying assumption that the MPPI ``optimal'' plan acts as a one-step feedback controller correcting for previous prediction errors. Under this assumption, if the approximate dynamics are linear-Gaussian, $\xi\ge 1$, and the system has enough control authority, we can use a procedure similar to \cite{knuth2021planning} to show that the one-step errors resulting from applying the feedback law are always exchangeable with the errors observed at the first planning step. Then, by applying the conformal calibration procedure derived for the first planning step (Theorem \ref{thm:cal1step} of the main paper), we show how we can extend the first-step safety guarantees for the remaining planning steps. We empirically observe that LUCCa can achieve higher safety rates (as shown in the experiments), since it is generally over-conservative compared to this simplified one-step feedback law.

Recall the \textit{unknown} time-invariant map $f$ describing the true system dynamics, previously presented in Equation \eqref{eq:dynamics} of the main paper
\begin{equation*}
    s_{t+1} \sim f(s_t, u_t)
\end{equation*}
Without loss of generality, we will now rewrite $s_{t+1} = f(s_t, u_t, \eta)$,
where $\eta$ is a stochastic disturbance drawn from an unknown time-invariant distribution.
The approximate dynamics $\tilde f$ is defined to propagate a Multivariate Normal (MVN) estimate of the true state, which we denote as $\mathcal N(\hat \mu_{t+1}, \hat \Sigma_{t+1}) = \tilde f(\hat \mu_t, \hat \Sigma_t, u_t)$.
Note that both the true $f$ and approximate dynamics $\tilde f$ take in the same control input $u_t$.

Let us now define the prediction error at time $t$ as $e_t \coloneqq s_t - \hat \mu_t$,
where $s_t$ is the true system state and $\hat \mu_t$ the expectation of the approximate dynamics prediction. 
Following Assumption \ref{assump:dataset} of the main paper, $e_0=0$ and hence we can define the first-step error as
\begin{align*}
    e_1 &= s_1 - \hat \mu_1 \\
    &= f(s_0, u_0, \eta) - \mathbb E[\tilde f(\hat \mu_0, \Sigma_0, u_0)]\\ 
    &= f(s_0, u_0, \eta) - \mathbb E[\tilde f(s_0, \Sigma_0, u_0)]
\end{align*}
where $(\hat \mu_0,\Sigma_0)$ originates from the fixed mapping $g$ defined in Equation \eqref{eq:state2MVNmap}
and $\hat \mu_0 = s_0$ since $e_0=0$. 
This error accounts for both model mismatch (epistemic uncertainty) and noise (aleatoric uncertainty).
In the main paper, we assumed access to a dataset of first-step transitions $D_{cal}=\{(s_0, u_0, s_{1})^{(i)}\}_{i=1}^{\lvert D_{cal} \rvert}$
from which we computed the Mahalanobis distance between the approximate MVN $\mathcal N (\hat \mu_{1}, \hat \Sigma_{1})$ and the true state $s_{1}$ - as detailed in $\S$\ref{sec:calibrate}.
Then, in Theorem \ref{thm:cal1step}, we used conformal prediction to find a scalar $\xi \in \mathbb R$ to calibrated the approximate MVNs, such that, averaged over the distribution of states and actions,
\begin{equation*}
    \mathbb P(e_1^\top  (\xi \hat \Sigma_{1})^{-1}e_1 \le \chi^2_{\dim(\mathcal S),\alpha}) \ge (1-\alpha)
\end{equation*}
\looseness-1This guarantee provided that, after the first planning step, $(1-\alpha)\%$ of the probability mass of the true states $s_1$ were contained in the hyperellipsoid
corresponding to the $(1-\alpha)\%$ confidence region of the calibrated approximate dynamics prediction $(\hat \mu_1, \xi \hat \Sigma_1)=\tilde f_{cal}(\hat \mu_0, \Sigma_0, u_0)$.
We now aim to derive similar probability guarantees for the remainder of the planning horizon.
Concretely, we aim to construct prediction regions whose $(1-\alpha)\%$ confidence ellipsoids contain $(1-\alpha)\%$ of the probability mass of the true states $s_t$ for $t\ge 2$.
Jointly with MPC-enforced safe region constraints, this would enable us to provide probabilistic safety along the entire plan - as described by Equation \eqref{eq:alternative_safety} of the main paper.

While in the first planning step the approximate $\tilde f$ and true dynamics $f$ rollout from the same state $s_0 = \hat \mu_0$,
this is not necessarily the case for the following steps.
In general, $s_t \ne \hat \mu_t$ for $t\ge 1$ and so, from the second step onwards, the approximate and true dynamics rollouts will have different starting points.
This can cause an accumulation of one-step errors, resulting in $e_t$ for $t \ge 2$ not being necessarily exchangeable with $e_1$, and hence the conformal scaling factor $\xi$ being possibly insufficient to ensure coverage. If we could modify the control inputs $u_t$ to account for previous errors such that $e_t$ became exchangeable with $e_1$, then the first-step calibrated covariance $\hat \Sigma_{1,cal}$ would provide the desired coverage at every step (by Theorem \ref{thm:cal1step}). We will show how this might be achieved if $\tilde f$ is linear-Gaussian, the system has sufficient control authority, and the planner acts as a one-step feedback controller. 
\begin{assumption}
    The approximate dynamics $\tilde f$ are linear-Gaussian, i.e. \\$\mathcal(\hat \mu_{t+1}, \hat \Sigma_{t+1}) = \tilde f(\hat \mu_t, \hat \Sigma_t, u_t) = (\tilde A \hat \mu_t + \tilde B u_t, \tilde A \hat \Sigma_t \tilde A^\top + \tilde Q)$, where $\tilde A, \tilde B, \tilde Q$ are matrices of approximate dimension and $\tilde Q \ge 0$.
\end{assumption}
Note that the true unknown dynamics $f$ may still be nonlinear and non-Gaussian.
The calibrated approximate dynamics update $\tilde f_{cal}$, defined in Algorithm 2, can then be rewritten as
\begin{align*}
\hat\mu_{t+1} &= \tilde A \hat \mu_t +\tilde B u_t \\
\hat \Sigma_{t+1, cal} &= \xi(\tilde A \hat \Sigma_{t, cal}\tilde A^\top + \tilde Q) 
\end{align*}
where $\hat \Sigma_{0,cal} = \Sigma_0$ as set by the map $g$. We can now rewrite the first-step error as
\begin{align*}
    e_1 &= s_1 - \hat \mu_1 \\
    &= f(s_0, u_0, \eta) - \mathbb E[\tilde f(\hat \mu_0, \Sigma_0, u_0)] \\
    &= f(s_0, u_0, \eta) - (\tilde A \hat \mu_0 + \tilde B u_0) \\
    &= f(s_0, u_0, \eta) - (\tilde A s_0 + \tilde B u_0)
\end{align*}

\begin{figure}[t]
\centering
\includegraphics[width=0.7\textwidth]{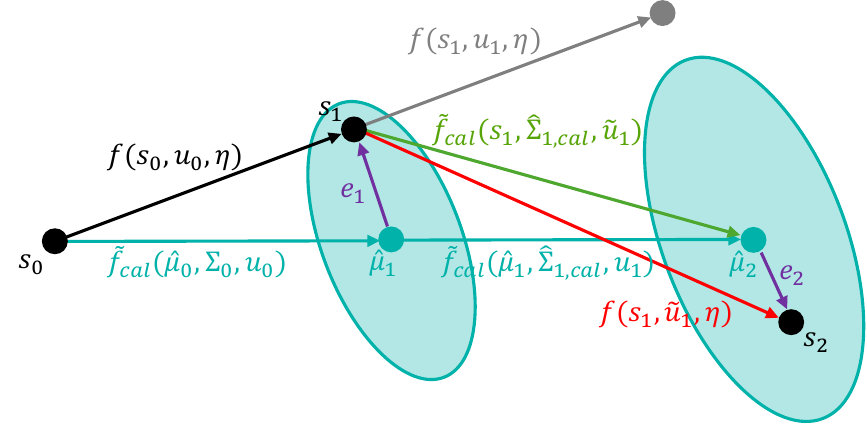}
\caption{First two planning steps with conformal calibration and the one-step feedback law.
The true system evolves according to $f$ with examples of possible states shown in black.
The ``optimal'' plan is shown in teal, along with the $(1-\alpha)\%$ confidence regions obtained from recursively applying $\tilde f_{cal}$.
After the first planning step, if we keep applying the ``optimal'' control inputs $u_t$ in an open-loop fashion, the error (purple) between the true system state and the approximate dynamics can accumulate - case shown in light gray.
Instead, by applying the ``corrected'' control input $\tilde u_1$ obtained via the one-step feedback law (green and red), we ensure that the error at times $t\ge 2$ remains exchangeable with the one-step error.
Hence, we can maintain the original first-step coverage guarantees for the remaining planning steps when such a feedback law exists.
}
\label{fig:multi_step_u_correct}
\end{figure}
Similarly to \cite{knuth2021planning}, let us now analyze the conditions for the existence of a one-step feedback law that corrects for the drift $e_t$ between the planned and executed trajectories at every step.
While in our ``optimal'' plan, each control input $u_t$ progresses the expectation of the approximate dynamics from $\hat \mu_t$ to $\hat \mu_{t+1}$,
during execution, the system may be at a different state $s_t \ne \hat \mu_t$. 
A one-step feedback law would sense the current state $s_t$ and apply a ``corrected'' control input $\tilde u_t \coloneqq u_t + \Delta u_t$
that drives the planned system from $s_t$ to the next planned point $\hat \mu_{t+1}$, i.e.
find $\tilde u_t$ such that $ \mathbb E [\tilde f_{cal}(s_t, \hat \Sigma_{t,cal}, \tilde u_t)]= \hat \mu_{t+1} \coloneqq \mathbb E[\tilde f_{cal}(\hat \mu_t, \hat \Sigma_{t,cal}, u_t)]$. A visualization of this process is shown in Figure \ref{fig:multi_step_u_correct}. Such a law would prevent the accumulation of error along the planned trajectory.

Since $\tilde f$ is linear, modifying the control input via the one-step feedback law will not change the predicted covariance $\hat \Sigma_t$.
Given this objective, we can solve for the required $\Delta u_t$ at each step using
\begin{align}
    \mathbb E [\tilde f_{cal}(s_t, \hat \Sigma_{t,cal}, \tilde u_t)] &= \mathbb E [\tilde f_{cal}(\hat \mu_t, \hat \Sigma_{t,cal},  u_t)]  \nonumber\\
    \tilde A s_t + \tilde B \tilde u_t &= \tilde A \hat \mu_t + \tilde B u_t  \nonumber\\
    \tilde A (\hat \mu_t + e_t) + \tilde B (u_t + \Delta u_t) &= \tilde A \hat \mu_t + \tilde B u_t \nonumber\\
    \tilde A \hat \mu_t + \tilde B u_t + \tilde A e_t + \tilde B \Delta u_t  &= \tilde A \hat \mu_t + \tilde B u_t \nonumber\\
    \Aboxed{\tilde A e_t + \tilde B \Delta u_t &= 0}
\end{align}
If $\dim(\mathcal U) \ge \dim(\mathcal S)$ and $\tilde B$ is full rank, then we can use the Moore-Penrose pseudo-inverse of $\tilde B$, denoted $\tilde B^{\dagger}$, to obtain
$\Delta u_t = - \tilde B^{\dagger} \tilde A e_t$ as the solution to the equation above.
Then, we can write the one-step feedback law as
\begin{equation*}
    \tilde u_t = u_t - \tilde B^{\dagger} \tilde A (s_t - \hat \mu_t)
\end{equation*}
We now make the simplifying assumption that the MPPI-computed plans correct for previous prediction errors, allowing us to make all the one-step errors exchangeable with $e_1$.
\begin{assumption}
    The ``optimal'' plan obtained by MPPI acts as the one-step feedback law defined above, correcting for previous prediction errors. 
\end{assumption}

By applying the one-step feedback law during inference, we have that error at time $t\ge 2$ then becomes
\begin{align*}
    e_{t+1} &= s_{t+1} - \hat \mu_{t+1} \\
    &= s_{t+1} - \mathbb E [\tilde f_{cal}(\hat \mu_t, \hat \Sigma_{t,cal},  u_t)] \\
    &= s_{t+1} - (\tilde A \hat \mu_t + \tilde B u_t)\\
    &= s_{t+1} - (\tilde A \hat \mu_t + \tilde B u_t + \smash{\overbracket{\tilde A e_t + \tilde B \Delta u_t}^{=0}})\\
    &= s_{t+1} - (\tilde A \hat \mu_t + \tilde A (s_t -\hat \mu_t) + \tilde B (u_t + \Delta u_t)) \\
    &= f(s_t, \tilde u_t, \eta) - (\tilde A s_t + \tilde B \tilde u_t)
\end{align*}
which has the same form as the first-step error, since the approximate and true dynamics have the same inputs.
Hence, by applying the one-step feedback law, we can make $e_t, t\ge2$  exchangeable with $e_1$.
From Theorem \ref{thm:cal1step}, it then follows that the $(1-\alpha)\%$ confidence region of the MVN centred at $\hat \mu_{t}$ with covariance $\hat \Sigma_{1,cal}=\xi\hat \Sigma_1$
will contain at least $(1-\alpha)\%$ of the probability mass of the true states $s_t$. This allows us to extend LUCCa's one-step safety guarantees to be valid across multiple steps.
\begin{theorem}
    Let $\varepsilon_{\hat{\mathcal N}(\hat{\mu}_{t}, \hat{\Sigma}_{1,cal}),\alpha}$ denote the $(1-\alpha)\%$ confidence hyperellipsoid of the MVN centered at $\hat \mu_t$ with covariance $\hat \Sigma_{1,cal}$. If the planner acts as a one-step feedback controller, $\dim(\mathcal U )\ge\dim(\mathcal S)$, and $\tilde B$ is full rank, then the 
    true states $s_t$ are contained in the $(1-\alpha)\%$ confidence region of $\mathcal N(\hat \mu_t, \hat \Sigma_{1,cal})$ with probability at least $(1-\alpha)$ at each time-step $t\in \{1,\ldots, T\}$.
\end{theorem}
\begin{proof}
    For $t=1$, this is ensured by conformal prediction as shown in Theorem \ref{thm:cal1step} of the main paper.
    For $t\ge 1$, if $\dim(\mathcal U )\ge\dim(\mathcal S)$ and $\tilde B$ is full rank then $\tilde u_t$ exists and is exact. By sensing $s_t$, we can calculate $\tilde u_t$ at every step.
    If the planner acts as a one-step feedback controller producing $\tilde u_t$, we can apply the input to the true and approximate systems, correcting for $e_t$ and making $e_{t+1}$ exchangeable with $e_1$.
    Then by Theorem \ref{thm:cal1step}, it follows that $\mathbb P(s_{t+1} \in \varepsilon_{\hat{\mathcal N}(\hat{\mu}_{t+1}, \hat{\Sigma}_{1,cal}),\alpha}) \ge (1-\alpha)$.
\end{proof}
Recall that LUCCa's calibrated predicted covariance at time $t$ is $\tilde \Sigma_{t,cal} = \sum_{i=0}^{t} \xi_1^i \tilde A^i \Sigma_0 (\tilde A^i)^\top + \sum_{i=0}^{t-1} \xi_1^{i+1} \tilde A^i \tilde Q (\tilde A^i)^\top$, which is generally overly conservative relative to the one-step calibrated covariance $\hat \Sigma_{1,cal}$ if $\xi\ge 1$. This conservative compounding of calibration factors was chosen to mitigate the impact of the planner possibly not acting as an exact one-step feedback controller.  
If the system has control input bounds, then one might also need to consider whether $\tilde u_t$ is admissible at each point along the plan. The proof above used a single global scaling factor $\xi$. We leave an extension guaranteeing multi-step local coverage for future work.

%% file: main.bbl
\begin{thebibliography}{10}
\providecommand{\url}[1]{\texttt{#1}}
\providecommand{\urlprefix}{URL }
\providecommand{\doi}[1]{https://doi.org/#1}

\bibitem{devCDC21}
Agrawal, D.R., Panagou, D.: Safe control synthesis via input constrained control barrier functions. In: CDC (2021)

\bibitem{angelopoulos2021gentle}
Angelopoulos, A.N., Bates, S.: A gentle introduction to conformal prediction and distribution-free uncertainty quantification. arXiv:2107.07511  (2021)

\bibitem{breiman2017classification}
Breiman, L.: {Classification and Regression Trees}. Routledge (2017)

\bibitem{kchua2018}
Chua, K., Calandra, R., McAllister, R., Levine, S.: {Deep Reinforcement Learning in a Handful of Trials using Probabilistic Dynamics Models}. In: NeurIPS (2018)

\bibitem{cosner2023robust}
Cosner, R.K., Culbertson, P., Taylor, A.J., Ames, A.D.: Robust safety under stochastic uncertainty with discrete-time control barrier functions. arXiv:2302.07469  (2023)

\bibitem{dtree_brazil}
Cruz~Cabezas, L., Otto, M.P., Izbicki, R., Stern, R.: {Regression Trees for Fast and Adaptive Prediction Intervals}. arXiv:2402.07357  (2024)

\bibitem{dixit23a}
Dixit, A., Lindemann, L., Wei, S.X., Cleaveland, M., Pappas, G.J., Burdick, J.W.: {Adaptive Conformal Prediction for Motion Planning among Dynamic Agents}. In: L4DC (2023)

\bibitem{guan2019}
Guan, L.: {Localized conformal prediction: a generalized inference framework for conformal prediction}. Biometrika  (2022)

\bibitem{han2022split}
Han, X., Tang, Z., Ghosh, J., Liu, Q.: Split localized conformal prediction. arXiv preprint arXiv:2206.13092  (2022)

\bibitem{simar19}
Hardle, W.K., Simar, L.: {Applied Multivariate Statistical Analysis}. Springer Cham, Switzerland (2019)

\bibitem{johnson88}
Johnson, R.A., Wichern, D.W.: {Applied Multivariate Statistical Analysis}. Pearson Prentice Hall, USA (1988)

\bibitem{knuth2021planning}
Knuth, C., Chou, G., Ozay, N., Berenson, D.: Planning with learned dynamics: Probabilistic guarantees on safety and reachability via lipschitz constants. IEEE Robotics and Automation Letters  (2021)

\bibitem{lei_condImpossible2014}
Lei, J., Wasserman, L.: {Distribution-free Prediction Bands for Non-parametric Regression}. J.R. Statist. Soc. B: Statistical Methodology  (2013)

\bibitem{lindemann2023safe}
Lindemann, L., Cleaveland, M., Shim, G., Pappas, G.J.: {Safe Planning in Dynamic Environments using Conformal Prediction}. IEEE RA-L  (2023)

\bibitem{schoellig19ECC}
McKinnon, C.D., Schoellig, A.P.: Learning probabilistic models for safe predictive control in unknown environments. In: ECC (2019)

\bibitem{askHelp23}
Ren, A.Z., Dixit, A., Bodrova, A., Singh, S., Tu, S., Brown, N., Xu, P., Takayama, L., Xia, F., Varley, J., Xu, Z., Sadigh, D., Zeng, A., Majumdar, A.: Robots that ask for help: Uncertainty alignment for large language model planners. In: Proceedings of the Conference on Robot Learning (CoRL) (2023)

\bibitem{loiannoTRO24}
Saviolo, A., Frey, J., Rathod, A., Diehl, M., Loianno, G.: Active learning of discrete-time dynamics for uncertainty-aware model predictive control. IEEE TRO  (2024)

\bibitem{shafer2008tutorial}
Shafer, G., Vovk, V.: A tutorial on conformal prediction. JMLR  (2008)

\bibitem{shi19_lander}
Shi, G., Shi, X., O’Connell, M., Yu, R., Azizzadenesheli, K., Anandkumar, A., Yue, Y., Chung, S.J.: {Neural Lander: Stable Drone Landing Control Using Learned Dynamics}. In: ICRA (2019)

\bibitem{sun2022copula}
Sun, S., Yu, R.: {Copula Conformal Prediction for Multi-Step Time Series Forecasting}. In: ICLR (2024)

\bibitem{tibshirani2019conformal}
Tibshirani, R.J., Foygel~Barber, R., Candes, E., Ramdas, A.: Conformal prediction under covariate shift. NeurIPS  (2019)

\bibitem{pmlr-v25-vovk12}
Vovk, V.: Conditional validity of inductive conformal predictors. In: Asian Conference on Machine Learning (2012)

\bibitem{vovk2005algorithmic}
Vovk, V., Gammerman, A., Shafer, G.: Algorithmic learning in a random world. Springer (2005)

\bibitem{mppi}
Williams, G., Wagener, N., Goldfain, B., Drews, P., Rehg, J.M., Boots, B., Theodorou, E.A.: {Information theoretic MPC for model-based reinforcement learning}. In: ICRA (2017)

\bibitem{xu2015CBF}
Xu, X., Tabuada, P., Grizzle, J.W., Ames, A.D.: Robustness of control barrier functions for safety critical control. IFAC  (2015)

\end{thebibliography}
